%% file: main.tex
\documentclass{article}
\usepackage{CJKutf8}
\usepackage{multirow}
\usepackage{booktabs}
\usepackage{url}            
\usepackage{booktabs}       
\usepackage{nicefrac}       
\usepackage{microtype}      
\usepackage{lipsum}
\usepackage{bbm}
\usepackage{optidef}
\usepackage{subcaption}
\usepackage{bm}
\usepackage{upgreek}
\usepackage{wrapfig}
\usepackage{enumitem,kantlipsum}
\usepackage{microtype}
\usepackage{setspace}

\usepackage{blindtext}
\usepackage{multirow} 
\usepackage{amsmath, amsthm, amsfonts, amssymb}
\newtheorem{asmp}{Assumption}
\usepackage{subcaption}
\usepackage{algpseudocode}
\usepackage[ruled,vlined]{algorithm2e}
\usepackage{tikz}
\usetikzlibrary{fit, backgrounds, positioning} 
\usepackage{tcolorbox}

\newcommand{\delequal}{\stackrel{\triangle}{=}}

\PassOptionsToPackage{numbers, compress}{natbib}

\usepackage[preprint]{neurips_2023}

\usepackage[utf8]{inputenc} 
\usepackage[T1]{fontenc}    
\usepackage{hyperref}       
\usepackage{url}            
\usepackage{booktabs}       
\usepackage{amsfonts}       
\usepackage{nicefrac}       
\usepackage{microtype}      
\usepackage{xcolor}         
\input{preamble}

\title{Distribution-consistency Structural Causal Models}
\author{%
  Heyang Gong 
    \\
  Kuaishou Inc., Beijing, China\\
  \texttt{gongheyang03@kuaishou.com} \\
  \And
  Chaochao Lu \\
  Shanghai AI Laboratory, Shanghai, China \\
  \texttt{luchaochao@pjlab.org.cn} \\
  \AND
  Yu Zhang \\
  Indiana University, Indiana, US \\
  \texttt{yzh10@iu.edu} \\
}

\begin{document}
\begin{CJK}{UTF8}{gbsn}  
\maketitle

\begin{abstract}
In the field of causal modeling, potential outcomes (PO) and structural causal models (SCMs) stand as the predominant frameworks. However, these frameworks face notable challenges in practically modeling counterfactuals, formalized as parameters of the joint distribution of potential outcomes. Counterfactual reasoning holds paramount importance in contemporary decision-making processes, especially in scenarios that demand personalized incentives based on the joint values of $(Y(0), Y(1))$. This paper begins with an investigation of the PO and SCM frameworks for modeling counterfactuals. Through the analysis, we identify an inherent model capacity limitation, termed as the ``degenerative counterfactual problem'', emerging from the consistency rule that is the cornerstone of both frameworks. To address this limitation, we introduce a novel \textit{distribution-consistency} assumption, and in alignment with it, we propose the Distribution-consistency Structural Causal Models (DiscoSCMs) offering enhanced capabilities to model counterfactuals. Furthermore, we provide a comprehensive set of theoretical results about the ``Ladder of Causation'' within the DiscoSCM framework. To concretely reveal the enhanced model capacity, we introduce a new identifiable causal parameter, \textit{the probability of consistency}, which holds practical significance within DiscoSCM alone, showcased with a personalized incentive example. We hope it opens new avenues for future research of counterfactual modeling, ultimately enhancing our understanding of causality and its real-world applications. 


\end{abstract}

\section{Introduction}

In the field of causal modeling, there are two primary frameworks: potential outcomes (PO) \citep{{rubin1974estimating, neyman1923application}} and structural causal models (SCMs) \citep{pearl2009causality}. In the PO framework, potential outcomes are generally recursively defined by their direct causes \citep{richardson2013single}, while in the SCM framework, counterfactual variables are defined as the solution of submodels induced by an intervention. Potential outcomes and counterfactual variables represent the core concepts in these two theoretically equivalent frameworks \citep{pearl2009causality, pearl2011consistency,  weinberger2023comparing},  each with different emphases.  It is important to note that in this paper,  our focus is neither on dissecting the specific differences nor on comparing the relative merits of the two frameworks. Instead, we selectively integrate concepts from both to enhance understanding and clarity.\footnote{For a detailed comparison, we recommend reading \citet{imbens2020potential} and other related sources \citep{markus2021causal, weinberger2023comparing}.} Counterfactual modeling refers to modelling the joint distribution of potential outcomes and valuating quantities related to this joint distribution.\footnote{The concept of counterfactuals is marked by ambiguity. While scholars like Judea Pearl define them as probabilistic answers to ``what-if'' scenarios, the term is broadly used in areas like counterfactual machine learning or reasoning, covering various techniques in causal analysis.}

Causal queries are categorized into three layers - associational, interventional, and counterfactual, collectively known as the Pearl Causal Hierarchy (PCH) \citep{pearl2019seven} and formalized with Layer valuations \citep{bareinboim2022pearl}. Counterfactuals, i.e. Layer 3 valuations, are intrinsically linked to the joint distribution of counterfactual outcomes and play a crucial role in personalized decision-making processes \citep{mueller2023personalized}. For instance, the probability of a counterfactual event $(Y(t)=y, Y(t')=y')$ is of significant interest, where $T$ denotes the decision to issue a discount coupon and $Y$ represents the act of making a purchase. This counterfactual probability aids in determining which individuals should be given a discount coupon. However, although counterfactuals are essentially important research questions, they are notoriously challenging to answer.

Surprisingly, we discover that a fundamental assumption in causal inference, \textit{the consistency rule} \citep{cole2009consistency, vanderweele2009concerning} \footnote{It is an assumption in the PO framework, but it serves as a theory within the SCM framework\citep{pearl2011consistency}, hence we prefer to refer it as to the consistency rule.}, which allows us to identify causal quantities, might be the exact reason we encounter difficulties in modeling counterfactuals. For example, in the binary treatment scenario, we consider the counterfactual distribution \( P(Y(0), Y(1)) \). For those who have an observed treatment $T=t$ and outcome $Y=y$, $Y(t)$ is equal to a deterministic constant $y$ whether $t$ equals 0 or 1, according to the consistency rule. This directly makes one component of the random vector $(Y(0), Y(1))$ degenerate to a constant, resulting in additional intricacies for computation, e.g., calculating the probability of being a complier $P(Y(0)=0, Y(1)=1)$. This leads to the \textit{degenerative counterfactual problem}, which is exactly the reason of inability to directly model the joint distribution of counterfactual outcomes. In fact, there are objections rooted in the philosophy of determinism to potential responses as having real existence \citep{dawid2023personalised}, which is a premise for consistency to make sense.


To address this challenge and move beyond determinism, we propose the ``distribution-consistency assumption''. This is a relaxation of the consistency rule, while preserving its strength in causal identification. In alignment with this new assumption, we introduce the Distribution-consistency Structural Causal Models (DiscoSCMs), and further establish a theory on Layer valuations within this framework to tackle the ``Ladder of Causation'' more practically.
We reveal and prove that at Layers 1 and 2, the valuations in DiscoSCMs and SCMs are equivalent.
However, their valuations at Layer 3 are generally different, except for some specific circumstances, e.g., when the counterfactual noise equals the factual noise. To showcase the enhanced capabilities of this framework in counterfactual modeling, we introduce a novel causal parameter, \textit{the probability of consistency}. This is a counterfactual valuation that holds practical significance only within DiscoSCMs, demonstrated with a concrete personalized incentive example. 


To sum up, our main contributions are as follows. 

\begin{itemize}
    \item We identify the degenerative counterfactual problem in existing causal frameworks based on the consistency rule, alongside with a concise formulation of them and a relevant scenario for industrial decision-making. 
    \item We propose the distribution-consistency assumption and theoretically reveal its implications for causal identification.
    \item We propose the DiscoSCM framework to address the degenerative counterfactual problem, which not only serves as an extension of both PO and SCM framework, but also circumvents the concerns on real existence of potential responses.    
    \item  We establish a comprehensive theory for Layer valuations within DiscoSCM, incorporating a novel lens on individual and population-level considerations. The enhanced capacity of modeling counterfactuals is validated and illustrated with both synthetic and practical examples. 
\end{itemize}




\section{Preliminaries} 

\subsection{Causal Frameworks Based on the Consistency Rule}
\label{appendix:framework}

There are two main frameworks for causal models: potential outcomes (PO) \citep{rubin1974estimating, neyman1923application} and structural causal models (SCMs) \citep{pearl2009causality}, both of which are theoretically equivalent frameworks based on the consistency rule \citep{pearl2011consistency}. Notably, they face challenges even when computing simple counterfactuals, such as analyzing the correlation between counterfactual outcomes with and without aspirin \citep{pearl2022direct}. The PO framework can be interpreted as experimental approach on causality with clear individual semantics \citep{weinberger2023comparing}, while the SCM framework emphasizes  the underlying causal mechanisms and derives a three-layer of causal information: associational, interventional and counterfactual \cite{pearl2009causality, bareinboim2022pearl}.

Following the formulation in \citet{imbens2020potential}, the PO approach begins with a population of units. There is a treatment/cause $T$ that can take on different values for each unit. Corresponding to each treatment value, a unit is associated with a set of potential outcomes, represented as $Y(t)$. For example, in the simplest case with a binary treatment there are two potential outcomes, $Y(0)$ and $Y(1)$. Only one of these potential outcomes, corresponding to the treatment received,  can be observed. Formally,
\begin{equation}
\label{eq:y_obs}
Y= T Y(1) + (1 - T) Y(0).
\end{equation}
This equation is direct derivation of the consistency rule: 
\begin{asmp}[\textbf{Consistency} \citep{angrist1996identification, imbens2015causal}] The potential outcome $Y(t)$ precisely matches the observed variable $Y$ given observed treatment $T=t$, i.e.,
\begin{align}
\label{assump:consist}
    T=t \Rightarrow Y(t) = Y.
\end{align}
\end{asmp}
The causal effect is related to the comparison between potential outcomes, of which at most one corresponding realization is available, with all the others missing. \citet{holland1986statistics} refers to this missing data nature as the ``fundamental problem of causal inference''. This framework is distinguished by its explicit semantics at the individual level. It is reflected in the stable unit treatment value assumption (SUTVA) \citep{rubin1974estimating, neyman1923application}, which ensures that there is no interference among individuals and the treatment given to one unit does not influence the outcomes of other units. 

The framework of \textit{structural causal models} (SCMs) is presented as follows.  

\begin{definition}[\textbf{Structural Causal Models} \citep{pearl2009causality}]
    \label{def:scm} 
    A structural causal model is a tuple $\langle \mathbf{U}, \mathbf{V}, \mathcal{F}\rangle$, where 
    \begin{itemize}
        \item $\mathbf{U}$ is a set of background variables, also called exogenous variables, that are determined by factors outside the model, and $P(\cdot)$ is a probability function defined over the domain of $\mathbf{U}$; 
        \item $\mathbf{V}$ is a set $\{V_1, V_2, \ldots, V_n\}$ of (endogenous) variables of interest that are determined by other variables in the model -- that is, in $\mathbf{U} \cup \mathbf{V}$;
        \item $\mathcal{F}$ is a set of functions $\{f_1, f_2, \ldots, f_n\}$ such that each $f_i$ is a mapping from (the respective domains of) $U_{i} \cup Pa_{i}$ to $V_{i}$, where $U_{i} \subseteq \*{U}$, $Pa_{i} \subseteq \mathbf{V} \setminus V_{i}$, and the entire set $\mathcal{F}$ forms a mapping from $\mathbf{U}$ to $\mathbf{V}$. That is, for $i=1,\ldots,n$, each $f_i \in \mathcal{F}$ is such that 
        $$v_i \leftarrow f_{i}(pa_{i}, u_{i}),$$ 
        i.e., it assigns a value to $V_i$ that depends on (the values of) a select set of variables in $\*U \cup \*V$.
    \end{itemize}
\end{definition}
SCMs are instrumental in modeling the underlying causal mechanisms behind phenomena, with these mechanisms being essential for answering causal questions. A typical example of such an inquiry might be: ``What if I take aspirin? will my headache be cured?'' \citep{pearl2019seven} A modification of an SCM allows for natural valuations of these kinds of quantities, which we will define as follows.
\begin{definition}[\textbf{Submodel-``Interventional SCM''} \citep{pearl2009causality}]
    Consider an SCM $\langle \mathbf{U}, \mathbf{V}, \mathcal{F}\rangle$, with a set of variables $\*X$  in $\*V$, and a particular realization $\*x$ of $\*X$. The $\doo(\*x)$ operator, representing an intervention (or action), modifies the set of structural equations $\mathcal{F}$ to $\mathcal{F}_{\*x} := \{f_{V_i} : V_i \in \*V \setminus \*X\} \cup \{f_X \leftarrow x : X \in \*X\}$ while maintaining all other elements constant. Consequently, the induced tuple $\langle \mathbf{U}, \mathbf{V}, \mathcal{F}_{\*x}\rangle$ is called as \textit{Intervential SCM} , and potential outcome $\*Y(\*x)$ (or denoted as $\*Y_{\*x}(\*u)$) is defined as the set of variables $\*Y \subseteq \*V$ in this submodel.
\end{definition}

Formally, an SCM gives valuation for associational, interventional and counterfactual quantities in the Pearl Causal Hierarchy (PCH) as follows.

\begin{definition}[\textbf{Layer Valuation} \citep{bareinboim2022pearl}]
\label{def:l3-semantics}
An SCM $\langle \*U, \*V, \mathcal{F}\rangle$ induces a family of joint distributions over potential outcomes $\*Y(\*x), \ldots, \*Z({\*w})$, for any $\*Y$, $\*Z$, $\dots$, $\*X,$ $\*W \subseteq \*V$: 
\begingroup\abovedisplayskip=0.5em\belowdisplayskip=0pt
\begin{align}\label{eq:def:l3-semantics}
    P(\*{y}_{\*{x}},\dots,\*{z}_{\*{w}}) =
\sum_{\substack{\{\*u\;\mid\;\*{Y}({\*x})=\*{y},\\ \dots,\; \*{Z}({\*w})=\*z\}}}
    P(\*u).
\end{align}
\endgroup
is referred to as Layer 3 valuation. In the specific case involving only one intervention, e.g., $do(\*x)$:
\begingroup\abovedisplayskip=0.5em\belowdisplayskip=0pt
\begin{align}
    \label{eq:def:l2-semantics}
    P({\*y}_{\*x}) = \sum_{\{\*u \mid {\*Y}({\*x})={\*y}\}}
    P(\*u),
\end{align}
\endgroup
is referred to as Layer 2 valuation. The even more specialized case when $\*X$ is empty:
\begingroup\abovedisplayskip=0.5em\belowdisplayskip=0pt
\begin{align}
    \label{eq:def:l1-semantics}
    P({\*y}) = 
    \sum_{\{\*u \mid {\*Y}={\*y}\}}
    P(\*u).
\end{align}
\endgroup
is referred to as Layer 1 valuation. Here, $\*y$ and $\*z$ represent the observed outcomes, $\*x$ and $\*w$ the observed treatments, $\*u$ the noise instantiation, and we denote $\*y_{\*x}$ and $\*z_{\*w}$ as the realization of their corresponding potential outcomes. 
\end{definition}

Each SCM induces a causal diagram $\mathcal{G}$, where every $V_i \in \*V$ is a vertex and there is a directed arrow $(V_j \rightarrow V_i)$ for every $V_i \in \*V$ and $V_j \in Pa(V_i)$.  In the case of acyclic diagrams, which correspond to recursive SCMs, \doo-calculus can be employed to completely identify all the valuations at Layer 2 \citep{pearl1995causal, huang2012pearl}. However, calculating counterfactuals at Layer 3 is generally far more challenging compared to Layers 1 and 2. This is because it essentially requires modeling the joint distribution of potential outcomes (e.g., the joint distribution of the potential outcomes with and without aspirin). Unfortunately, we often lack access to the underlying causal mechanisms and only have observed traces of them. This limitation renders the practical use of Eq.\eqref{eq:def:l3-semantics} highly restricted for computing counterfactuals.

\subsection{A Motivating Scenario of Personalized Incentives}

In online platforms, deploying incentives such as coupons or coins is a fundamental strategy to enhance user engagement and revenue. Nonetheless, these incentives, if not judiciously allocated, can be costly and may result in a suboptimal return on investment (ROI) \citep{goldenberg2020free, ai2022lbcf}. The core challenge resides in tailoring the incentive amount for each individual user within a predefined budget, taking into account the varied responses of users to these incentives. Our research, set in the background of a large-scale video platform, seeks to optimize returns within a constrained budget by leveraging data-driven, personalized incentive strategies. This task involves deciphering the impact of incentives on returns, framed as treatments and outcomes. We represent features by $\mathbf{X}$, treatment by $T$, and outcome by $Y$. A simplified illustrative example is as follows.

\begin{example}
\label{eg:incentive}
Consider a causal model tailored for personalized incentives on a population $U$, encompassing observable variables: $S$, $\mathbf{X}$, $T$, and $Y$, as depicted in the causal diagram (Figure \ref{fig:incentive}). The causal mechanisms for each user $U=u$ are described below:\footnote{In this paper, $U=u$ is interpreted as the index of a user, but actually it can be a learned unobservable causal representation in order to facilitate computations, hence denoted with gray color node in Fig. \ref{fig:incentive}.}
\begin{enumerate}
    \item $S$:  Assigns users to one of three experiment groups. The random group ($S=0$) receives incentives based purely on chance; the pure strategy group ($S=1$) has incentives tailored according to specific user characteristics; and the mixed strategy group ($S=2$) combines random allocation with user-specific strategies.
    \item $\mathbf{X}$: Denotes the pre-treatment features of the user that influence both the treatment and outcome. This includes demographic details, historical engagement levels, and other relevant factors. 
    \item $T$: A binary incentive treatment variable. For users in the random group ($S=0$), this decision is made with uniform probability; For users in the pure strategy group ($S=1$), the incentive is a deterministic function of the pre-treatment features; In the mixed strategy group ($S=2$), the decision is influenced by the user's features but retains some randomness.
    \item $Y$: The outcome variable of the user's reaction to the incentive, e.g. conversion, purchase or retention.
\end{enumerate}
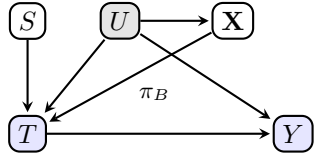
\begin{figure}[http]
\centering
    \begin{tikzpicture}[->, >=stealth, shorten >=1pt, auto, thick, node distance=1.5cm, main node/.style={rectangle, rounded corners, draw}]
    \node[main node] (T) {$S$};
    \node[main node, opacity=0.9, fill=blue!10] (A) [below of=T] {$T$};
    \node[main node, opacity=0.9, fill=blue!10] (Y) [right=3cm of A] {$Y$}; 
    \node[main node, opacity=0.9, fill=gray!20] (Z) [above right=1cm and 0.7cm of A] {$U$}; 

    \node[main node] (X) [right of=Z] {$\mathbf{X}$};
    \path[every node/.style={font=\sffamily\small}]
        (T) edge node {} (A)
        (Z) edge node {} (A)
        (Z) edge node {} (X)
        (Z) edge node {} (Y)
        (A) edge node {} (Y)
        (X) edge node {$\pi_B$} (A);
        
    \end{tikzpicture}
    \caption{Causal model for personalized incentives: this diagram illustrates the causal relationships among group assignment $S$, incentive treatment $T$, pre-treatment features $\mathbf{X}$, and the outcome variable $Y$. The model integrates a user representation $U$, capturing all relevant endogenous information (excluding $T$) that determines the Layer valuations regarding to $(T, Y)$. $\pi_B$ stands for a personalized incentive policy that is designed to optimize returns under total  budget constraint $B$. 
    }
\label{fig:incentive}
\end{figure}
In a practically motivating setting, $T$ encompasses $K$ distinct treatments $\{1, 2, ..., K\}$, with \( T=0 \) representing the control group. For each user \( u \) and treatment \( j \), we denote by \( c_{uj} \) the treatment cost, and by \( B \) the overall budget constraint. Our objective is to optimize the treatment selection policy 
$$\pi_B: \*x_u \to t,$$ 
so as to strategically assign treatments to maximize the total return within the budget $B$. This policy hinges on estimating the conditional average treatment effect (CATE) for each user, i.e., 
\begin{equation*}
\tau_{j}(\mathbf{x}_u)=\mathbb{E}[Y(j)-Y(0)|\mathbf{X}=\mathbf{x}_u], 
\end{equation*}
where $\mathbf{x}_u$ represents observed features for user $u$. Note that in subsequent sections, for the ease of understanding, we prefer to use terminology specific to this motivating scenario.
\end{example}



\section{The Degenerative Counterfactuals Problem}

The challenge of identifying individuals with a desired response pattern is pervasive across diverse domains, including industry \citep{ai2022lbcf}, marketing \citep{tsai2009customer}, and health science \citep{sarvet2023perspectives, mueller2023perspective},  which underscores the importance of estimating counterfactuals. It is particularly relevant in providing personalized incentives (e.g., coupons at Amazon or discounts at Uber) used by online platforms to drive user engagement and increase revenue. In scenarios involving binary treatments and outcomes, users are typically classified into four behavioral categories \citep{angrist91g, balke1997bounds}: compliers, always-takers, never-takers, and defiers, based on their reactions to different treatments \footnote{Compliers will respond positively only if encouraged ($Y(0)=0, Y(1)=1$), always-takers always respond positively regardless of encouragement ($Y(0)=1, Y(1)=1$), never-takers always respond negatively ($Y(0)=0, Y(1)=0$), and defiers respond negatively if encouraged and positively otherwise ($Y(0)=1, Y(1)=0$).}. Hence, estimating the probability of counterfactual events becomes vital. 

Uplift modeling \citep{zhao2019uplift, diemert2018large, gutierrez2017causal, zhang2021unified}, commonly viewed as a PO approach, is often used to address this challenge. This method essentially involves estimating the CATE $\tau(\*x_u) = \mathbb{E}[Y(1)-Y(0)|\*X=\*x_u]$, and then selects the users $u$ with high $\tau$ values. It can be seen as an A/B-test-based heuristic \citep{li2019unit} that estimates certain specific parameter of the joint distribution of potential outcomes, instead of directly modeling it. As such, this prevailing approach cannot account for the counterfactual nature of the desired behavior. Although some recent work has attempted to explore the nature within the SCM framework, their resulting counterfactual expressions are generally not identifiable \citep{bareinboim2022pearl} and thus one still resorts to heuristic approaches \citep{li2019unit, li2022unit}. Specifically, they derive tight bounds for the counterfactual expressions based on both experimental and observational data, and then use the midpoint of the bounds as a unit selection criterion.
Nevertheless, heuristic methods typically yield only sub-optimal solutions and often lack theoretical guarantees. This highlights the necessity for an approach to modelling the joint distribution of potential outcomes ($Y(0), Y(1)$), which is the essence of counterfactual reasoning. It can radically address the limitations of the heuristic methods. However, an intriguing question emerges: Why has the existing literature typically avoided direct modeling of the joint distribution of potential outcomes?
 


To answer this question, it is crucial to understand that counterfactual queries inherently operate at the individual level. For example, the question - ``If someone took a drug and died, would they have survived had they not taken it?'' - deals with individual-specific outcomes. In the case of binary treatment, according to the consistency rule, for each $u$, either $Y_u(0)$ or $Y_u(1)$ aligns with an observed outcome $y$, depending on the treatment $t$. Consequently, the joint distribution of $(Y_u(0), Y_u(1))$ becomes degenerate.\footnote{In the PO and SCM frameworks, $Y_u(t)$ for $t=0, 1$ is viewed as a fixed value, leading to a Dirac distribution and degeneracy. Outside these frameworks, critiques like \citet{dawid2023personalised} challenge the existence of such potential responses. Aligning with this perspective, we posit that $Y_u(t)$ should be considered a random variable, acknowledging that degeneracy persists under the consistency rule.} Similarly, for a sub-population in which each unit has the observed treatment $T=t$ and outcome $Y=y$ whether $t$ equals 0 or 1, $Y(t)$ is equal to a deterministic constant $y$, according to the consistency rule. This directly causes that one component of the random vector $(Y(0), Y(1))$ degenerates to a constant. We refer to the above phenomena as \textbf{the degenerative counterfactual problem}. This degeneration imposes additional intricacies, within conventional causal frameworks, for computing counterfactuals, e.g., the probability of being a complier $P(Y(0)=0, Y(1)=1)$.

This dilemma prompts a crucial question:  Should we adhere to the consistency rule, or is there a relaxation of this rule to avoid the degeneration issue? Consider a toy exam scenario: ``If an individual with average ability scores exceptionally high on a test due to good fortune, what score would the individual achieve had he retaken the test under the identical conditions? An exceptionally high score or an average one?'' Intuitively, predicting an average score seems more reasonable. Similarly, in the context of personalized incentive scenario: ``If a high amount of incentives and a high retention for a user are observed, back to the past, what retention would we expect had this user been given the high incentive again?'' Intuitively, predicting a high retention with a probability of less than 100\% makes more sense, due to uncontrollable randomness (e.g., good fortune). These examples indicate that a more flexible assumption, one that deviates from strict adherence to the consistency rule, may be more aligned with practical realities. Hence, we propose a distribution consistency assumption to overcome the counterfactuals modeling limitation, which serves as a replacement and relaxation of the consistency rule.

\subsection{The Distribution-consistency Assumption}

The consistency rule forms the cornerstone of mainstream causal modeling frameworks \citep{neyman1923application, rubin1974estimating, imbens2015causal, pearl2009causality, vanderweele2009concerning, cole2009consistency}. It establishes a connection between potential outcomes and observed data. Within the PO approach, this rule is treated as an assumption, while in the SCM framework, it is considered as a theorem \citep{pearl2011consistency}. The consistency rule posits that if an observed treatment $X$ equals $x$, then the potential outcome $Y(x)$ will precisely match the observed outcome $Y$. The consistency rule as a mathematical instrument facilitates the identification of causal quantities, i.e., translating probability statements involving counterfactuals into those concerning the ordinary conditional probabilities of observed variables. However, this paper posits that the consistency rule might be over-strict, thus resulting in the degenerative counterfactual problem as aforementioned. We suggest that the observed value $y$ of outcome merely represents one sample of the counterfactual outcome $Y(x)$ given the observed treatment $X=x$, leading to the following assumption: 
\begin{asmp}[\textbf{Distribution-consistency}]
\label{assump:distri-consist}
For any individual represented by $U=u$ with an observed treatment $X = x$, the counterfactual outcome $Y(x)$ is equivalent in distribution to the observed outcome $Y$. Formally, 
\begin{equation}
\label{eq:assump:distri-consist}
X = x, U=u \Rightarrow Y(x) \stackrel{d}{=} Y
\end{equation}
where \(\stackrel{d}{=}\) denotes equivalence in distribution.
\end{asmp}
To illustrate the distinction between the consistency and distribution-consistency assumptions, let us revisit the toy exam scenario mentioned previously.  Predicting a high score on the retest would indicate adherence to the consistency rule. By contrast, anticipating an average score aligns with the distribution-consistency assumption, suggesting a probabilistic view of counterfactual outcome. In other words, our assumption complies with the answer - an average score, that seems more practical since good fortune is intuitively uncontrollable.
A direct derivation from the distribution-consistency assumption is as follows:
\begin{lemma}
\label{lemma:ident}
For binary treatment $X$ and each individual $u$:
\begin{align}
Y \stackrel{d}{=} X Y(1) + (1 - X) Y(0).
\end{align}
\end{lemma}

\begin{proof} Under the distribution-consistency assumption, for any $y$:
\begin{align*}
P(Y = y | X = 1, U = u) &= P(Y(1) = y | X = 1, U = u) \\
&= P(X Y(1) + (1 - X) Y(0) = y | X = 1, U = u)
\end{align*}
and similarly,
\begin{align*}
P(Y = y | X = 0, U = u) &= P(Y(0) = y | X = 0, U = u) \\
&= P(X Y(1) + (1 - X) Y(0) = y | X = 0, U = u)
\end{align*}
Using these two equations, we obtain:
\begin{align*}
&P(Y = y | U = u) \\
&= P(Y = y | X = 1, U = u) P(X = 1 | U = u) + P(Y = y | X = 0, U = u) P(X = 0 | U = u) \\
&= P(X Y(1) + (1 - X) Y(0) = y | X = 1, U = u) P(X = 1 | U = u) \\
&\quad + P(X Y(1) + (1 - X) Y(0) = y | X = 0, U = u) P(X = 0 | U = u) \\
&= P(X Y(1) + (1 - X) Y(0) = y | U = u).
\end{align*}
\end{proof}

This lemma corresponds to Eq. \eqref{eq:y_obs}, indicating that although the distribution-consistency assumption is a relaxation of the consistency rule, many conclusions for causal identification as a derivation of the consistency assumption has its counterpart under the distribution-consistency assumption.





Introducing the more relax distribution-consistency assumption to the PO framework aims at overcoming the limitations in model capacity regarding the estimation of counterfactuals. Nevertheless, the PO framework, without structural equations as primitives, has no rigorous definitions of Layer valuations that are capable of conveniently expressing probabilities of counterfactuals \citep{bareinboim2022pearl}. We are hence motivated in the next section to propose a novel framework, Distribution-consistency Structural Causal Model (DiscoSCM), which generalizes SCMs, to tackle the modeling of counterfactuals. More surprisingly, we find that the distribution-consistency assumption can be naturally derived as a theorem from the DiscoSCM framework. 


\section{Distribution-consistency Structural Causal Models}


To address the challenge of the degenerative counterfactual issue, we formally propose: 

\begin{definition}[\textbf{Distribution-consistency Structural Causal Model (DiscoSCM)}]
    \label{def:discoscm}
    A DiscoSCM is a tuple $\langle  U, \mathbf{E}, \mathbf{V}, \mathcal{F}\rangle$, where
    \begin{itemize}
        \item $U$ is a unit selection variable, where each instantiation $U=u$ denotes an individual. It is associated with a probability function $P(u)$, uniformly distributed by default.
        \item $\mathbf{E}$ is a set of exogenous variables, also called noise variables, determined by factors outside the model. It is independent to $U$ and associated with a probability function $P(\*e)$;
        \item $\mathbf{V}$ is a set of endogenous variables of interest $\{V_1, V_2, \ldots, V_n\}$, determined by other variables in $\mathbf{E} \cup \mathbf{V}$;
        \item $\mathcal{F}$ is a set of functions $\{f_1(\cdot, \cdot; u), f_2(\cdot, \cdot; u), \ldots, f_n(\cdot, \cdot; u)\}$, where each $f_i$ is a mapping from $E_{i} \cup Pa_{i}$ to $V_{i}$, with $E_{i} \subseteq \mathbf{E}$, $Pa_{i} \subseteq \mathbf{V} \setminus V_{i}$, for individual $U=u$. Each function assigns a value to $V_i$ based on a selected set of variables in $\mathbf{E} \cup \mathbf{V}$. That is, for $i=1,\ldots,n$, each $f_i(\cdot, \cdot; u) \in \mathcal{F}$ is such that 
        $$v_i \leftarrow f_{i}(pa_{i}, e_i; u),$$ 
        i.e., it assigns a value to $V_i$ that depends on (the values of) a selected set of variables in $\*E \cup \*V$ for each individual $U=u$. 
    \end{itemize}
\end{definition}
The fundamental difference between DiscoSCMs and SCMs is the explicit incorporation of individual semantics in DiscoSCMs, represented by $U$. This addition allows for the construction of a novel intervention, corresponding to the \textit{do}-operator in SCMs, which accounts for the uncontrollable nature of the noise variable value. 
\begin{definition}[\textbf{Do-operator}]
    For a DiscoSCM $\langle  U, \mathbf{E}, \mathbf{V}, \mathcal{F}\rangle$, $\*X$ is a set of variables in $\*V$ and $\*x$ represents a realization, the $do(\mathbf{x})$ operator modifies: 1) the set of structural equations $\mathcal{F}$ to 
        \begin{align*}
            \mathcal{F}_{\mathbf{x}} := \{f_i : V_i \notin \mathbf{X}\} \cup \{\mathbf{X} \leftarrow \mathbf{x}\},
        \end{align*}
        and; 2) noise $\mathbf{E}$ to couterfactual noise $\mathbf{E}(\mathbf{x})$ maintaining the same distribution. \footnote{Note that $\mathbf{E}(\mathbf{x})$ is not a function of $\*x$, but rather a random variable indexed by $\*x$. Importantly, it shares the same distribution as $\mathbf{E}$.} The induced submodel $\langle U, \mathbf{E}(\*x), \mathbf{V}, \mathcal{F}_{\mathbf{x}}\rangle$ is called the \textit{interventional DiscoSCM}.
\end{definition}
It should be emphasized that in SCMs, the counterfactual and factual worlds possess identical exogenous noise values, whilst in DiscoSCMs, the counterfactual noise $\mathbf{E}(\mathbf{x})$ and factual noise $\mathbf{E}$ are only assumed to follow identical distributions. Now, we can further define the core concept of counterfactual outcomes in the realm of DiscoSCMs.
\begin{definition}[\textbf{Counterfactual Outcome}]
    For a DiscoSCM $\langle  U, \mathbf{E}, \mathbf{V}, \mathcal{F}\rangle$, $\*X$ is a set of variables in $\*V$ and $\*x$ represents a realization. The counterfactual outcome $\*Y^d(\*x)$ (or denoted as $\*Y(\*x)$, $\*Y_{\*x}(\*e_{\*x})$ when no ambiguity concerns) is defined as the set of variables $\*Y \subseteq \*V$ in the submodel $\langle U, \mathbf{E}(\*x), \mathbf{V}, \mathcal{F}_{\mathbf{x}}\rangle$. In the special case that $\*X$ is an empty set, the corresponding submodel is denoted as $\langle U, \mathbf{E}^d, \mathbf{V}, \mathcal{F}\rangle$ and its counterfactual noise and outcome as $\*E^d$ and $\*Y^d$, respectively.
\end{definition}

\begin{example}
    Consider a causal model that describes a system where \(X\) causes \(Y\). In the SCM framework, this is formalized as:
    \begin{equation}
    \left\{
    \begin{aligned}
        X &= f_1(U_1), \\
        Y &= f_2(X, U_2).
    \end{aligned}
    \right.
    \end{equation}
    In contrast, in the DiscoSCM framework, it's formalized as:
    \begin{equation}
    \left\{
    \begin{aligned}
        X &= f_1(E_1; U), \\
        Y &= f_2(X, E_2; U).
    \end{aligned}
    \right.
    \end{equation}
    In SCMs, each instantiation \(\*U = \*u\) defines a sample \citep{pearl20217}, while in DiscoSCMs, a sample is defined by the combination of \(U = u\) and \(\*E = \*e\). Thus, DiscoSCMs explicitly include \(U\) semantics, similar to the PO approach, enhancing flexibility and generalizability. As a specific case, with \(U, E_1, E_2\) as stardard normally distributed variables, and \(f_1, f_2\) as simple linear functions:
    \begin{equation}
    \left\{
    \begin{aligned}
        X &= U + E_1, \\
        Y &= X + U + E_2.
    \end{aligned}
    \right.
    \end{equation}
    The interventional DiscoSCM model, induced by \(do(X=x)\), is given by:
    \begin{equation}
    \left\{
    \begin{aligned}
        X &= x, \\
        Y &= x + U + E_2(x).
    \end{aligned}
    \right.
    \end{equation}
    Consequently, for any given individual \(u\), the counterfactual outcome 
    \begin{align}
    \label{eq:unit_po}
        Y^d_u(x) = x + u + E_2(x),
    \end{align}
    where \(E_2(x)\) denotes the counterfactual noise maintaining the same distribution as \(E_2\). Importantly, this unit counterfactual outcome remains a random variable, whereas in SCMs, unit potential responses are assumed to be a fixed value. In fact, the counterfactual outcome $Y_u^d(x)$ in DiscoSCMs, while setting $\*E(x) = \*E$, should be consistent with $Y(x)$ in SCMs, since observed $\*E=\*e$ and $U=u$ together determine a sample. 
\end{example}
This example concretely demonstrate that, in the DiscoSCM framework, the value of counterfactual outcome $\*Y^d(\*x)$ depends on instantiation of counterfactual noise $\*E(\*x)$  rather than the factual noise $\*E$. DiscoSCMs decouple the randomness in data into two independent components: the individual semantics represented by unit variable $U$ and exogenous variable  $\*E$ for each unit. In other words, DiscoSCMs address the randomness of selecting units akin to the PO approach, while managing the exogenous uncertainty of the underlying causal generative process, akin to SCMs. This makes it a hybrid of both frameworks. Analogous to the consistency rule served as a theorem in the SCM framework, our DiscoSCMs have the counterpart property as follows. 
\begin{lemma}[\textbf{Distribution-consistency Rule}]
Consider a DiscoSCM with treatment $X=f_1(E_1; U)$ and outcome $Y=f_2(X, E_2; U)$. 
For any individual \(u\) and treatment $X=x$, the following holds:
    \begin{equation}
    \label{eq:assump:distri-consist_new}
    X = x, U=u \Rightarrow Y^d(x) \stackrel{d}{=} Y
    \end{equation}
where $Y^d(x)$ denotes the counterfactual outcome under treatment $x$. 
\end{lemma}
\begin{proof} By the definitions and the probability formula, we have:
    \begin{align*}
        X = x, U=u &\Rightarrow Y^d(x) \stackrel{d}{=} Y \\
        \Leftrightarrow P(Y^d(x) = y|X=x,U=u) &=  P(Y = y|X=x,U=u) \\
        \Leftrightarrow P(Y^d(x) = y, X=x,U=u) &= P(Y = y, X=x,U=u)  \\
        \Leftrightarrow P(f_2(x, E_2(x); u) = y, f_1(E_1; u)=x,U=u) &= P(f_2(x, E_2; u) = y, f_1(E_1; u)=x,U=u)  \\
        \Leftrightarrow P(f_2(x, E_2(x); u) = y)P(f_1(E_1; u)=x)P(u) &= P(f_2(x, E_2; u) = y)P(f_1(E_1; u)=x)P(u)\\
        \Leftrightarrow P(f_2(x, E_2(x); u) = y) &= P(f_2(x, E_2; u) = y)
    \end{align*}
    Given the identical distribution between $E_2(x)$ and $E_2$, the above equality holds, completing the proof.
\end{proof}
Substituting the consistency rule with the distribution-consistency rule in DiscoSCMs reflects the inherent uncontrollability in the instantiation process of the noise variable. In the previously mentioned exam scenario, this uncontrolled random noise enables the prediction of an average score for an individual with average ability who, due to good luck, scores exceptionally high on a test. For practical application scenario of personalized incentives, observing a high amount of incentives and high retention for a user  does not guarantee that the DiscoSCM will predict the same high retention if high incentives are replicated under identical conditions. 
This distinctive feature of DiscoSCMs can be formalized through the introduction of a novel concept as follows.
\begin{definition}[\textbf{Probability of Consistency (PC)}] 
\label{def:pc}
For any individual $u$, 
\begin{eqnarray}
\text{PC}(u)\delequal P(y_x|y,x; u) = P(Y^d(x) = y |X = x, Y=y, U=u)
\label{ps}
\end{eqnarray}
\end{definition}
Notably, it is evident that PC degenerates to the constant 1 in the SCM framework due to the consistency rule. Henceforth, it is a parameter that only holds significance within the DiscoSCM framework, which will be concretely shown in the personalized incentive example (Example \ref{eg:pc}). Personalized decision-making necessitates the estimation of heterogeneous causal parameters across individuals, e.g., $\tau_j(\*x_u)$ in our industrial practical application scenario. Proponents of the SCM framework are thus confronted with a pertinent question:  \textit{Why is there a distinction across individuals assuming i.i.d. (independently and identically distributed) samples? }

In the framework of SCMs, this question is addressed by considering the different instantiations of  \(\*U = \*u\),  which represent both which individual and noise value. Hence, it suggests that different noise realizations result in heterogeneity across individuals, which violates the intuition that noise and causal parameters should be conceptually independent. In contrast, the DiscoSCM framework attributes heterogeneity to each individual instantiation \(U = u\), which can be interpreted as its causal representation, thereby aligning with the intuitive independence between noise $\*E$ and causal parameters. In other words, causal parameters are functions of the individual instantiation \(u\), independent of any specific exogenous noise value $\*e$. Therefore, it is the explicit individual representation in DiscoSCMs that offers us a more nuanced and precise understanding of individual heterogeneity, which is the core difference between SCMs and DiscoSCMs.

\paragraph{Interpretation of the Surrogate Paradox.} In fact, this interpretation of heterogeneity in DiscoSCM can readily aid in explaining the well-known surrogate paradox \cite{chen2007criteria, weinberger2023comparing} — a scenario where a treatment has a positive effect on the surrogate and the surrogate has a positive effect on the outcome, yet the treatment does not have a positive effect on the outcome. Specifically, let us consider a DiscoSCM with treatment \( T \), features \( X \), surrogate $S$ and outcome \( Y \), along with its structural equations:
\begin{equation*}
    \left\{
    \begin{aligned}
        T &= g_1(X, E_1; U) \\
        S &= (T + g_2(X, E_2) )  \mathbf{1}_{U \in A} \\
        Y &= (S + g_3(X, E_3) )\mathbf{1}_{U \in A'},
    \end{aligned}
    \right.
\end{equation*}
where $g_i, i=1, 2, 3$ are given functions, \( A \) represents a subpopulation, and \( A' \) is the complement of \( A \). Therefore, 
$$Y = ((T + g_2(X, E_2) )  \mathbf{1}_{U \in A} + g_3(X, E_3) )\mathbf{1}_{U \in A'} = g_3(X, E_3) \mathbf{1}_{U \in A'}.$$
Obviously, in this model, $T$ has positive effect on the surrogate $S$, and $S$ has positive effect on $Y$. However, $T$ has no causal effect on $Y$. This example reveals the power of our DiscoSCM framework that explicitly incorporates the individual semantic $U$ when investigating the transitivity of causality. 



\section{Layer Valuations in the DiscoSCM Framework}

Just as SCM is highly appreciated for the formalization of causal queries across the PCH, DiscoSCM similarly introduces a formalized approach to evaluating associational, interventional and counterfactual quantities. 

\subsection{Main Theoretical Results}

\begin{definition}[\textbf{Layer Valuation with DiscoSCM}]
\label{def:semantics_new}
A DiscoSCM $\langle U, \mathbf{E}, \mathbf{V}, \mathcal{F}\rangle$ induces a family of joint distributions over counterfactual outcomes $\*Y(\*x), \ldots, \*Z({\*w})$, for any $\*Y$, $\*Z, \dots, \*X$, $\*W \subseteq \*V$:
\begin{align}\label{eq:def:l3-semantics_new}
    P(\*{y}_{\*{x}},\dots,\*{z}_{\*{w}}; u) =
\sum_{\substack{\{\*e_{\*x}\, ...,  \*e_{\*w}\;\mid\;\*{Y}^d({\*x})=\*{y},\;\;\;\dots,\; \*{Z}^d({\*w})=\*z, U=u\}}}
    P(\*e_{\*x}, ..., \*e_{\*w}).
\end{align}
is referred to as Layer 3 valuation. In the specific case involving only one intervention \footnote{When \( \*X = \emptyset \), we simplify the notation \( \*Y^d(\*x) \) to \( \*Y^d \) and \( \*E_{\*x} \) to \( \*E^d \).
}, e.g., $\doo(\*x)$:
\begin{align}
    \label{eq:def:l2-semantics_new}
    P({\*y}_{\*x}; u) = 
    \sum_{\{\*e_{\*x} \;\mid\; {\*Y}^d({\*x})={\*y}, U=u\}}
    P(\*e_{\*x}),
\end{align}
is referred to as Layer 2 valuation. The case when no intervention:
\begin{align}
    \label{eq:def:l1-semantics_new}
    P({\*y}; u) = 
    \sum_{\{\*e \;\mid\; {\*Y}={\*y}, U=u\}}
    P(\*e),
\end{align}
is referred to as Layer 1 valuation. Here, $\*y$ and $\*z$ represent the observed outcomes, $\*x$ and $\*w$ the observed treatments, $\*e$ the noise instantiation, $u$ the individual, and we denote $\*y_{\*x}$ and $\*z_{\*w}$ as the realization of their corresponding counterfactual outcomes, $\*e_{\*x}$, $\*e_{\*w}$ as the instantiation of their corresponding counterfactual noises.
\end{definition}

Note that, Layer valuations in the SCM and DiscoSCM frameworks differ primarily because of the explicit individual semantic $U$ and its modified $do$ operator. To examine the relationship between Layer valuations under the two frameworks, we first present the following theorem:
\begin{theorem}
\label{thm:relation_l12}
For the same system described within the SCM and DiscoSCM frameworks, each valuation at Layer 1 or 2 in both frameworks is equivalent.
\end{theorem}

\begin{proof}
    In both frameworks, any valuation $P(\*y)$ at Layer 1 represents the probability of the same observed outcome variable $\*Y$ equal to $\*y$, hence they are equal. 

    For Layer 2 valuations, under the assumption of an identical distribution between the counterfactual noise \( \mathbf{E}(\mathbf{x}) \) and the factual noise \( \mathbf{E} \) in the DiscoSCM framework, we have:
    $$P(\*e_{\*x}) \triangleq P(\*E(\*x) = \*e) = P(\*E = \*e) \triangleq P(\*e)$$
    Since the underlying causal mechanisms of the same system described by SCM and DiscoSCM are the same, the structural equations \( \mathcal{F}_{\mathbf{x}} \) of the submodel induced by \( do(\mathbf{x}) \) in both SCM and DiscoSCM should align. Consequently, in DiscoSCM, utilizing the corresponding identical structural equations, the equality 
    $$\sum_{\{\*e_{\*x} \;\mid\; {\*Y}^d({\*x})={\*y}, U=u\}} P(\*e_{\*x}) = \sum_{\{\*e \;\mid\; {\*Y}({\*x})={\*y}, U=u\}} P(\*e)$$
    holds for any individual \( u \). Therefore, 
    \begin{align*}
        P(\*Y^d(\*x) = \*y) &= \sum_u P(\*Y^d(\*x)=\*y|U=u) P(U=u)  \\
        &= \sum_u \sum_{\{\*e_{\*x} \;\mid\; {\*Y}^d(\*x)={\*y}, U=u\}}  P(\*e_{\*x}) P(u) \\
        &= \sum_u \sum_{\{\*e \;\mid\; {\*Y(\*x)}={\*y}, U=u\}}  P(\*e) P(u) \\
        &=  \sum_{\{\*e, u \;\mid\; {\*Y_{\*x}(\*e)}={\*y}, U=u\}}  P(\*e, u) 
    \end{align*}
    This formula proves the Layer 2 equivalence within the two frameworks, because the instantiation \( (\mathbf{e}, u) \) in DiscoSCM and instantiation \( \mathbf{u} \) in SCM both represent the exogenous value for generation process of data points in the same system.
\end{proof}

Therefore, from the perspective of Layer 1 and 2 valuations, there is no difference between the SCM and DiscoSCM frameworks. At Layer 3, however, there is generally no equivalence, except for the special case stated below. 
\begin{theorem}
\label{thm:relation_l3}
    For the same system described with the SCM and DiscoSCM frameworks, each valuation at Layer 3 in both frameworks is equivalent when randomness of counterfactual outcomes stems entirely from the selection of units.
\end{theorem}

\begin{proof}
    For any treatment $\*X$ and outcome $\*Y$, given that the randomness of counterfactual outcomes is entirely due to the selection of units, the counterfactual outcome $\*Y^d(\*x)$ for each individual $U=u$ should be a fixed value, rather than a random variable. Recognizing that $\*Y^d(\*x)$ is a function of \(\*E(\*x)\), it generally necessitates a Dirac distribution concentrated at a single value. Combined with the identical distribution relation of factual noise and counterfactual noise, i.e. \(\*E \stackrel{d}{=}\*E(\*x)\), we deduce that 
    \[\*E = \*E(\*x).\] 
    Consequently, the factual noise $\*U=\*u$ in SCMs aligns with counterfactual noise $(\*E(\*x)=\*e, U=u)$ in DiscoSCMs. This implies that the causal generation process of both $\*Y^d(\*x)$ and $\*Y(\*x)$ have equivalent exogenous values as inputs. Considering the same structural equations in both frameworks that describe causal mechanisms of the same system, we therefore conclude that:
    $$\*Y^d(\*x) = \*Y(\*x)$$
    This equivalence naturally leads to the alignment of all Layer valuations, including Layer 3 valuation, in both frameworks.
\end{proof}

We now turn to explain the difference at Layer 3 valuation in details. For an SCM with treatment $X$ and outcome $Y$, to perform Layer 3 valuation in the form of $P(Y(x)=y|e)$, \citet{pearl2009causality} proposes a three-step process:
\begin{enumerate}
    \item Abduction: Update the probability \( P(u) \) to obtain \( P(u|e) \).
    \item Action: Modify the equations determining the variable \( X \) to \( X = x \).
    \item Prediction: Utilize the modified model to compute the probability \( P(Y = y) \).
\end{enumerate}
Here, \( e \) represents the observed trace, also referred to as evidence, exemplified by $X=\tilde{x}, Y=\tilde{y}$, concisely denoted as $e = [ \tilde{x}, \tilde{y}]$.\footnote{Please distinguish between the ``evidence'' (\( e\)) and the instantiation of ``exogenous noise '' (\(\*e\)).}  In contrast, to correspondingly conduct Layer 3 valuation  $P(Y^d(x)=y|e)$ in DiscoSCMs, we have a theorem as follows.
\begin{theorem}[\textbf{Population-Level Valuations}]
\label{algo:population}
Consider a DiscoSCM wherein $Y^d(x)$ is the counterfactual outcome, and \(e\) represents the observed trace or evidence. The Layer 3 valuation \(P(Y^d(x)|e)\) is computed through the following process:

\textbf{Step 1 (Abduction):} Derive the posterior distribution \(P(u|e)\) of the unit selection variable \(U\) based on the evidence \(e\).

\textbf{Step 2 (Valuation):} Compute individual-level valuation \(P(y_x;u)\) in Def. \ref{def:semantics_new} for each unit \(u\).

\textbf{Step 3 (Reduction):} Aggregate these individual-level valuations to obtain the population-level valuation as follows:
\begin{equation}
\label{eq:population}
P(Y^d(x)=y|e) = \sum_u P(y_x;u) P(u|e),
\end{equation}
\end{theorem}

\begin{proof} The population-level valuation:
    \begin{align*}
        P(Y^d(x)=y|e) &= \sum_u P(Y^d(x)=y|e, U=u) P(U=u|e) \\
        &= \sum_u P(Y^d(x)=y|U=u) P(U=u|e) \\
        &\delequal \sum_u P(y|x;u) P(u|e)
    \end{align*}
    The first equality follows the total probability formula, and the second equality is due to Theorem \ref{lemma:unit_indep} that is derived from Assumption  \ref{asmp:indep}. This completes the proof.
\end{proof}


\begin{asmp}
\label{asmp:indep}
    In DiscoSCMs, counterfactual noises are independent of the factual noise, i.e., 
     \[\*E(\*x) \perp\!\!\!\perp \*E.\] 
     for realization $\*x$ of a set of variables $\*X$.
\end{asmp}
Considering that counterfactual noises are from counterfactual worlds, whereas the factual noise comes from the actual world, they inherently belong to different worlds and thus should be independent. Given Assumption \ref{asmp:indep}, we have:
\begin{lemma}
\label{lemma:unit_indep}
    Given an individual \(u\), the counterfactual outcome \(Y^d(x)\) is independent of the evidence \(e\), expressed as:
    \begin{align}
    \label{eq:lemma:unit_indep}
         P(Y^d(x)=y|e, U=u) = P(Y^d(x)=y|U=u)
    \end{align}
    where \(x, y\) denote the realizations of \(X, Y\) respectively. 
\end{lemma}

\begin{proof}
    Considering that \(e\) represents the evidence induced by observed variables, and these observed variables are defined as functions of the factual noise, \(e\) is therefore an event generated by the factual noise. For any given individual \(u\), the counterfactual outcome \(Y^d(x)\) is determined as a function of counterfactual noise \(E(x)\). Therefore, given Assumption \ref{asmp:indep} which states that counterfactual noises are independent of factual noise, the condition \(U=u\) leads to the independence of \(Y^d(x)\) from \(e\).
\end{proof}

Up to this point, we have obtained population-level valuations in the DiscoSCM framework. A natural question arises: is it possible to achieve individual-level valuations? Fortunately, this can be directly derived from Lemma \ref{lemma:unit_indep}, which is formally presented as below.
\begin{theorem}[\textbf{Individual-Level Valuations}]
\label{thm:layer12e}
    For any individual \(u\) and evidence \(e\), the following equivalence holds:
    \begin{align}
        P(y_x|e;u) = P(y_x;u) = P(y|x;u) 
    \end{align}  
\end{theorem}
\begin{proof}
    The first equality is Lemma \ref{lemma:unit_indep}. Regarding the second equality, if we set $e=\{X=x\}$ in Lemma \ref{lemma:unit_indep}  and apply the distribution-consistency rule, we obtain:
    \begin{align*}
        P(Y^d(x) = y|U=u) &= P(Y^d(x) = y|X=x, U=u) \\
        &= P(Y=y|X=x, U=u) \\
        &\delequal P(y|x;u)
    \end{align*}
    This completes the proof.
\end{proof}

It is noteworthy that there is no distinction between individual-level and population-level valuations in the SCM framework. This is evidenced by the statement in \citet{pearl20217} (p. 3) that in an SCM, the vector $\*U = \*u$ can be interpreted as an individual subject or experimental unit, uniquely determining the values of all domain variables. By contrast, the DiscoSCM framework fundamentally disentangles individual semantics from the exogenous variables, where the logic of unit selection is governed by a unit selection variable $U$, independent to the exogenous variable $\*E$.
Consequently, DiscoSCMs introduce a novel paradigm facilitating the separation of individual-level and population-level calculations. 

\subsection{Illustration with Examples} 

Now, we further illustrate the difference between both frameworks concretely with the probability of consistency (Def. \ref{def:pc}) for Layer 3 valuation. For ease of clarification, we first present the following result, which can be directly derived from Theorem \ref{thm:layer12e} by setting $e=[x, y]$.
\begin{corollary}
\label{lemma:pc}
    For any individual $u$, the probability of consistency
    \begin{align}
        PC(u) = P(y|x; u).
    \end{align}
\end{corollary}
\begin{example}[Continue with Example \ref{eg:incentive}]
\label{eg:pc}
    In the SCM framework, the probability of consistency degenerates to constant 1 for any unit $u$, due to the consistency rule. However, within the DiscoSCM framework, it is a parameter that holds significance. Specifically, following Example \ref{eg:incentive}, the probability of consistency $\text{PC}(u)$ can be obtained by Corollary \ref{lemma:pc}:
    \begin{equation}
    \label{eq:probconsistz2t}
         \text{PC}(u)\delequal P(t_s|t, s; u) = P(t|s;u) =
        \begin{cases}
          0.5 & \text{if}\ s_u = 0, \\
          1 & \text{if}\ s_u = 1, \\
          g(t, x) \in (0, 1) & \text{if}\ s_u = 2,
        \end{cases}   
    \end{equation}
    where $g(\cdot, \cdot)$ is a certain given function. In other words, for any user in the random group, the probability of consistency is 0.5, indicating the assigned treatment value is attributed to randomness. In the pure strategy group, the probability of consistency equals 1, implying the assigned value of $T$ is attributed to a deterministic strategy. For the mixed strategy group, the probability of consistency falls within the range $(0, 1)$. This indicates that the value of $T$ is attributed partly to the strategy and partly to randomness.
\end{example}

Moreover, we demonstrate our main results on Layer valuations at both individual and population level with a synthetic example.

\begin{example}
    Consider a population \( U = \{1, 2, \ldots, 200\} = S \cup S' \), with observed data 
    \[\{(t_1, y_1), (t_2, y_2), \ldots, (t_{200}, y_{200})\}, \] 
    where the first 100 units belong to \( S \) and the rest to \( S' \). The two subpopulations exhibit different causal effects, and the distribution of treatment and outcome for these 200 units is summarized in Table \ref{tab:eg_200units}.
    
    \begin{table}[h!]
        \centering
        \caption{Observed data for \((t_u, y_u)\) with 200 units, where numbers in blue represent units in \( S \), while those in black denote units in \( S' \).}
        \[
        \begin{array}{c|ccccc}
        \hline
        t_u \backslash y_u & 2 & 1 & 0 & -1 & -2 \\ \hline
        -1 & 0 & 0 & 25 & 25 + \textcolor{blue}{20} & \textcolor{blue}{80} \\
        1 & 25 & 25 & 0 & 0 & 0 \\
        \hline
        \end{array}
        \]
        \label{tab:eg_200units}
    \end{table}


    \textit{Individual-level Valuations.} Layer 1 valuation at the individual level is a common predictive learning problem. In typical scenarios, we have various features for each individual and often can train a machine learning model to accurately predict \( P(y, t) \) and \( P(t) \) for each individual \( u \). Consequently, Layer 1 valuation \( P(y|t;u) \) can be obtained using observational data. Furthermore, by Theorem \ref{thm:layer12e} , Layer 1, 2, and 3 valuations are equivalent at the individual level, thereby resolving the individual-level Layer valuations.

    Assuming the ground truth causal mechanisms (\(T\) causing \(Y\)) that generate the observational dataset are known or learned as:     
    \begin{equation}
    \label{eq:synthetic}
        \begin{cases}
            Y_u = 2T_u + E_u, \quad P(T_u=1) = 0, \quad E_u \sim B(0, 1, 0.2) & \text{if } u \in S, \\
            Y_u = T_u + E_u , \quad P(T_u=1) = 0.5, \quad E_u \sim B(0, 1, 0.5) & \text{if } u \in S'.
        \end{cases}
    \end{equation}
    Here, the treatment \( t_u \) takes values in \( \{-1, 1\} \) for any individual \( u \), and \(B(0, 1, p)\) denotes a binomial distribution with success probability \(p\). We can then theoretically deduce the individual-level valuation using Eq. \eqref{eq:synthetic}. Specifically, for \( u \in S \) where \( E_u \sim B(0, 1, 0.2) \):
    \begin{align*}
        P(y|t;u) &= P(Y=y|T=t, U=u) \\
                 &= P(2T+E=y|T=t,U=u) \\
                 &= P(E_u = y - 2t).
    \end{align*}
    Similarly, for \( u \in S' \) where \( E_u \sim B(0, 1, 0.5) \), the valuation is given by \( P(y|t;u) = P(E_u = y - t) \). Furthermore, we can easily calculate the individual treatment effects (ITE) for units in \( S \) and \( S' \), demonstrating heterogeneity treatment effects across sub-populations:

    \begin{equation*}
        ITE(u) \triangleq E[Y(1)] - E[Y(-1)] = 
        \begin{cases}
            4 & \text{if } u \in S, \\
            2 & \text{if } u \in S'.
        \end{cases}
    \end{equation*}
    
    \textit{Population-level Valuations.}  In the context of DiscoSCM, individual-level valuations are considered as primitives, while population-level valuations are derived from these primitives. The computation of layer valuations at the population level can be conducted using Theorem \ref{algo:population}. Consider the following Layer valuation:
   \begin{align*}
        P(Y=-1|T=-1) 
        &= \sum_u P(Y=-1|T=-1;u) P(u|T=-1) \qquad\qquad\qquad \textbf{(Layer 1)} \\
        &= \frac{1}{5} \times \frac{2}{3} + \frac{1}{2} \times \frac{1}{3} \\
        &= \frac{3}{10},
    \end{align*}
    where the probability measure \( P( \cdot |T=-1) \) assign each $u$ with probability $\frac{2}{3} * \frac{1}{200}$ if $u \in S$ else $\frac{2}{3} * \frac{1}{200}$. This theoretical computation is consistent with empirical estimates obtained from observed data:
    $$
    P(Y=-1|T=-1) = \frac{\# \text{ of cases where } y_u = -1 \text{ given } t_u = -1}{\text{\# of cases with } t_u = -1} = \frac{25+20}{50 + 100} = \frac{3}{10}, 
    $$
    where the numerator represents the count of units with \( y_u = -1 \) when \( t_u = -1 \), and the denominator is the total count of units with \( t_u = -1 \). Moreover, we can similarly derive various Layer valuation quantities as follows:
    
    \begin{align*}
        P(Y^d(t=-1) = -1) &= \frac{1}{5} \times \frac{1}{2} + \frac{1}{2} \times \frac{1}{2} &&= \frac{7}{20}, &&&\textbf{(Layer 2)} \\
        P(Y^d=-1|T=-1) &= \frac{1}{5} \times \frac{2}{3} + \left(\frac{1}{2} \times \frac{1}{2}\right) \times \frac{1}{3} &&= \frac{13}{60}, &&&\textbf{(Layer 3)} \\
        P(Y^d(t=-1) = -1|T=-1, Y=-1) &= \frac{1}{5} \times \frac{4}{9} + \frac{1}{2} \times \frac{5}{9} &&= \frac{11}{30}. &&&\textbf{(Layer 3)}
    \end{align*}
    It is noteworthy that for an SCM describing the same system, Layer 1 and 2 valuations in the SCM are the same as those in the DiscoSCM, as stated in Theorem \ref{thm:relation_l12}. However, the two Layer 3 valuations within the DiscoSCM, \( P(Y^d=-1|T=-1) \) and \( P(Y^d(t=-1) = -1|T=-1, Y=-1) \), exhibit discrepancies when compared to their counterparts in the SCM. Specifically, we observe the following differences:
    $$P(Y=-1|T=-1) = \frac{3}{10},$$ 
    and by consistency rule, 
    $$P(Y(t=-1) = -1|T=-1, Y=-1)=1.$$
    This contrast highlights the differences in Layer 3 valuations between the two frameworks.
    
\end{example}

\subsection{Further Exploration of the Unit Selection Variable}
So far, we have rigorously defined Layer valuations and proposed and proved theorems pertaining to both population-level and individual-level valuations in the DiscoSCM framework. We have also illustrated their implications and computation methods through specific examples. A notable distinction between DiscoSCMs and SCMs is the explicit inclusion of a variable \(U\) to represent individuals. Finally, we introduce two significant theoretical results for causal identification, elucidating the role of \(U\).

\begin{theorem}
    For a DiscoSCM with binary treatment \(T\), outcome \(Y\), we have:
    \begin{align}
        Y(t) \perp T | e(U)
    \end{align}
    where \(t\) is a given treatment value and \(e(\cdot)\) is the propensity function.
\end{theorem}

\begin{proof}  
    By Theorem \ref{thm:layer12e}, we have \(P(y_t|t;u) = P(y_t; u)\). Therefore:
    \begin{align*}
        Y(t) \perp T | U \Rightarrow E[T| Y(t), U] = E[T|U]
    \end{align*}
    Therefore, 
    \begin{align*}
        P(T=1|e(U), Y(t)) 
        &= E[T|e(U), Y(t)] \\
        &= E[E[T|e(U), Y(t), U]|e(U), Y(t)] \\
        &=  E[E[T| Y(t), U]|e(U), Y(t)] \\
        &= E[E[T|U]|e(U), Y(t)] \\
        &= E[e(U)|e(U), Y(t)] \\   
        &= e(U)
    \end{align*}
    Because this does not depend on \(Y(t)\),  we have proven that the conditional independence.
\end{proof}

\begin{theorem}
    For a DiscoSCM with discrete treatment \(T\) and outcome \(Y\), 
    \begin{align}
    \label{lemma6}
        E[Y(t^*)] = E\left[\frac{ Y \mathbf{1}_{T=t^*} }{ P(t^*|U)}\right]
    \end{align}
    where \(t^*\) is any given treatment value and \(P(t^*|u)\) is positive for any individual \(U=u\).
\end{theorem}
\begin{proof}
    Consider the left-hand side of Equation \ref{lemma6}. By decomposing \(E[Y(t^*)]\), we have:
    \begin{align*}
        E[Y(t^*)] &= E[E[Y(t^*)|U]] \\
        &= \sum_u \sum_y y P(y_{t^*}; u) P(u) \\
        &= \sum_u \sum_y y P(y|{t^*}; u) P(u) \\
        &= \sum_u \sum_y y \frac{P(y|t^*, u)P(t^*|u) P(u) }{P(t^*|u)}   \\
        &= \sum_u \sum_y y \frac{P(u, t^*, y) }{P(t^*|u)},
    \end{align*}
    Conversely, the right-hand side of Equation \ref{lemma6} is elaborated as follows:
    \begin{align*}
        E\left[\frac{ Y \mathbf{1}_{T=t^*} }{ P(t^*|U)}\right] 
        &= \sum_{u, t, y} \frac{y \mathbf{1}_{t=t^*} }{P(t^*|u)} P(u, t, y) \\
        &= \sum_{u, y} \sum_t \frac{y \mathbf{1}_{t=t^*} }{P(t^*|u)} P(u, t, y) \\
        &= \sum_{u} \sum_y \frac{y  }{P(t^*|u)} P(u, t^*, y) .
    \end{align*}
    By synthesizing these derivations, the proof is concluded.
\end{proof}

The above two theorems respectively correspond to Theorems 7.1\&7.18 in \cite{neal2020introduction} related to the propensity score in the PO framework. The key difference is that in our theorems, \(U\) replaces the pre-treatment features that satisfy the unconfoundedness assumption\citep{imbens2015causal}. This suggests that in our DiscoSCM framework, \(U\) plays a role analogous to that of pre-treatment features. Moreover, if we assume the conditional independence between $X$ and $(T, Y)$ given $U$ in the underlying causal graph,  the following theorems hold.
\begin{theorem}
\label{thm:iden_y_t_x}
    For a DiscoSCM with discrete treatment $T$, feature $X$ and outcome $Y$: 
    \begin{align}
        E[Y(t')|X=x'] = E[\frac{ Y \mathbf{1}_{T=t', X=x'} }{ P(t'|U) P(x')}] 
    \end{align}
    where $t'$ is any given treatment value and $P(t'|u)$ is positive for any individual $U=u$.
\end{theorem}
\begin{proof}
   On one hand,  we have:
    \begin{align*}
        E[Y(t')|X=x'] &= E[E[Y(t')|X=x', U]|X=x'] \\
        &= \sum_u \sum_y y P(y_{t'}|x'; u) P(u|x') \\
        &= \sum_u \sum_y y P(y|{t'}; u) P(u|x') \\
        &= \sum_u \sum_y y \frac{P(y|t', u)P(t'|u) P(u) }{P(t'|u)} \cdot \frac{P(x'|u, t', y)}{P(x'|u)}  \cdot \frac{P(x'|u)}{P(x')} \\
        &= \sum_u \sum_y y \frac{P(u, t', y) }{P(t'|u)} \cdot \frac{P(x'|u, t', y)}{P(x')} \\
        &= \sum_u \sum_y y \frac{P(u, t', y, x') }{P(t'|u)P(x')} 
    \end{align*}
    \footnote{The equation $P(x'|u, t', y) = P(x'|u)$ is derived by from the conditional independence.} The second equality holds by Theorem \ref{thm:layer12e}, while the other equalities is grounded in basic derivations and property of probability. 

    On the other hand, as the expectation function of random variables $T, X, Y, U$, 
    \begin{align*}
        E[\frac{ Y \mathbf{1}_{T=t', X=x'} }{ P(T=t'|U) P(X=x')}] 
        &= \sum_{u, t, y, x} \frac{y \mathbf{1}_{t=t', x=x'} }{P(t'|u) P(x')} P(u, t, y, x) \\
        &= \sum_{u, y} \sum_t \sum_x \frac{y \mathbf{1}_{t=t'} \cdot \mathbf{1}_{x=x'} }{P(t'|u)P(x')} P(u, t, y, x) \\
        &= \sum_{u, y} \sum_t \frac{y \mathbf{1}_{t=t'} }{P(t'|u)P(x')} P(u, t, y, x') \\
        &= \sum_u \sum_y \frac{y}{P(t'|u)P(x')} P(u, t', y, x')
    \end{align*}
    Combining the above two formulas, we complete our proof.
\end{proof}

\begin{theorem}
\label{thm:iden_y_t_x}
    For a DiscoSCM with discrete treatment $T$, feature $X$ and outcome $Y$: 
    \begin{align}
        E[Y(t)|X=x, T=t, Y=y] = E[\frac{ Y \mathbf{1}_{T=t, X=x} }{ P(t, x)} \cdot \frac{P(y|t; U)}{ P(y|t, x)} ] 
    \end{align}
    where $x, t, y$ are the observed values of features, treatment and outcome.
\end{theorem}
\begin{proof}
   On one hand,  we have:
    \begin{align*}
        &E[Y(t)|X=x, T=t, Y=y] \\
        &= E[E[Y(t)|X=x, T=t, Y=y, U]|X=x, T=t, Y=y] \\
        &= \sum_u \sum_{y^*} {y^*} P(y^*_{t}|x, t, y; u) P(u|x, t, y) \\
        &= \sum_u \sum_{y^*} {y^*} P({y^*}|{t}; u) P(u|x, t, y) \\
        &= \sum_u \sum_{y^*} {y^*} \frac{P(y^*|t, u)P(t|u) P(u) }{P(t|u)}  \cdot \frac{P(x, t, y|u)}{P(x, t, y)} \\
        &= \sum_u \sum_{y^*} y^* \frac{P(u, t, y^*) }{P(t|u)} \cdot \frac{P(x|u, t, y^*)}{P(x|u, t, y)} \cdot \frac{P(x, t, y|u)}{P(x, t, y)} \\
        &= \sum_u \sum_{y^*} y^* \frac{P(u, t, y^*, x) }{P(t|u)P(x, t, y)} P(t, y|u)
    \end{align*}
    The second equality holds by Theorem \ref{thm:layer12e}, while the other equalities is grounded in basic derivations and property of probability. 
    On the other hand, as the expectation function of random variables $T, X, Y, U$, 
    \begin{align*}
        E[\frac{ Y \mathbf{1}_{T=t, X=x} }{ P(t, x)} \cdot \frac{P(y|t; U)}{ P(y|t, x)} ] 
        &= \sum_{u, t^*, y^*, x^*} \frac{y^* \mathbf{1}_{t=t^*, x=x^*} }{P(t, x) } \cdot \frac{P(y|t; u)}{ P(y|t, x)} P(u, t^*, y^*, x^*) \\
        &= \sum_{u, y^*} \sum_{t^*} \sum_{x^*} \frac{y^* \mathbf{1}_{t=t^*} \cdot \mathbf{1}_{x=x^*} }{P(t, x)} P(u, t^*, y^*, x^*) \cdot \frac{P(y|t; u)}{ P(y|t, x)} \\
        &= \sum_{u, y^*} \sum_{t^*} \frac{y^* \mathbf{1}_{t=t^*} }{P(t, x)} P(u, t^*, y^*, x)  \cdot \frac{P(y|t; u)}{ P(y|t, x)} \\
        &= \sum_u \sum_{y^*} \frac{y^*}{P(t, x)} P(u, t, y^*, x) \cdot \frac{P(y|t; u)}{ P(y|t, x)}
    \end{align*}
    Combining the above two formulas, we complete our proof.
\end{proof}

\section{Conclusion}

When faced with the simplest question, ``Consider an individual with average ability who takes a test and, due to good luck, achieves an exceptionally high score. If this individual were to retake the test under identical external conditions, what score will he obtain? An exceptionally high score or an average score?'' your choice fundamentally determines whether you should employ the traditional causal modeling framework or the DiscoSCM. The traditional causal modeling framework relies on the consistency rule, which would assure the same exceptionally high score in this situation. In contrast, if you choose our framework for causal inference, you will end up predicting an average score with a high probability rather than an exceptionally high score. Furthermore, you can predict the probability of consistency for top students and average ability students achieving high scores on retaking a test, which will be different. This reveals the greater model capacity within our framework, offering a more nuanced understanding of outcomes and the external factors that influence them. 

In essence, our goal is to advance the field of causal inference by presenting a novel framework that addresses existing limitations in counterfactual modeling. By establishing the distribution-consistency assumption and crafting the corresponding DiscoSCM framework and methodologies, we aspire to ignite further exploration of counterfactual modeling, ultimately enhancing our understanding of causal relationships and their real-world applications. At least, this paper serves as a comprehensive introduction to DiscoSCM, marked by its precise formalization and clarity of concepts.

\subsection{Discussion}

\textbf{1. What does ``counterfactual'' mean and what is its essence?} Counterfactuals encompass three fundamental lens: event, variable, and query. Firstly, let us start with the most basic concept: the event. A counterfactual event, i.e., Layer 3 event, is an event associated with at least two variables residing in different worlds (i.e., submodel or interventional model). In contrast, the events in Layers 1 and 2 are related to a single world: the former pertains to the factual world, while the latter concerns any interventional world (including the case of an empty intervention). Secondly, from the viewpoint of variable, a counterfactual variable refers to the one within a counterfactual world. Lastly, from the PCH perspective, a counterfactual query, i.e., Layer 3 valuation, refers to the probability of a counterfactual event or the parameters of different variables from at least two worlds, e.g., correlation of $Y(0)$ and $Y(1)$.

\textbf{2. Consistency or distribution-consistency?} Such questions similar to the aforementioned test score example do not have a correct or incorrect answer, and the subtlety lies in the interpretation of ``under identical external conditions''. Current causal modeling frameworks treat good fortune as an external condition and thus would predict an exceptionally high score by the consistency rule. If one's choice is to believe that good fortune does not necessarily reoccur, upon reverting to the past to retake the test, this individual would most likely achieve an average score due to their average ability; in this case, abandoning the consistency assumption aligns with that choice. Advocates of consistency might contend: ``An empirically minded scientist might say that, once we have data, the fact that some potential outcomes must equal their observed values in the data is a good thing; it is the information we have gained from the data.'' However, we prefer distribution-consistency and would argue: ``An empirically minded scientist might prefer to maximize the likelihood of the observed values of variables in the data, rather than imposing equality constraints on them. Therefore, the observed value of potential outcomes in the data should be consistent with their distribution, rather than being strictly equated''. It is plausible to conjecture that Pearl might favor consistency over distribution-consistency given his advocacy for determinism \citep{pearl2009causality}.

\textbf{3. Actual existence of potential response and determinism?} From the philosophical perspective: Potential responses, fundamental to causal inference, indeed face serious philosophical objections regarding their actual existence. Specifically, this existence was challenged by Dawid \citep{geneletti2011defining, dawid2023personalised}, who argued that potential responses are far from fundamental and are entirely unnecessary. He reasoned, \emph{``Only if we take a fully Laplacean view of the universe, in which the future of the universe is entirely determined by its present state and the laws of Physics, does this make any sense at all—and even then, it is difficult to incorporate the whims of an unconstrained external agent who decides whether or not to give treatment, or to account for the effect of external conditions arising after treatment. Whether or not we believe in a deep-down deterministic universe, our predictions of the future can only be based on the limited information we do have at our disposal, and must necessarily be probabilistic.''} In our DiscoSCM, we have perfectly addressed Dawid's concerns by introducing counterfactual noises  $\*E(\*x)$, resulting in probabilistic unit counterfactual outcomes, such as \(Y^d_u(x)\) in Eq. \eqref{eq:unit_po}.  If one firmly believes in determinism and denies the distribution-consistency assumption, a question arises: does such adherence lead to more convenient and practical models? After all, all models are wrong, but some are useful \cite{box1979all}. Hence, our preference is to explore how a model, which might be incorrect, can be useful. While the philosophical correctness of determinism remains a subject of debate, we posit that deterministic modeling is unrealistic. A model serves as an abstraction of natural phenomena, and there exists irreducible computational complexity that cannot be compressed. DiscoSCM might function as a ``reducible pocket''\footnote{The term "reducible pocket" is indeed a concept proposed by Stephen Wolfram, who claimed finding pockets of reducibility is the story of science. We use this metaphor to describe the predictive capabilities of DiscoSCM on modeling complex underlying causal mechanisms.} for this irreducible reality, serving as a probabilistic causal modeling approach for deterministic causal ground-truth.


\bibliographystyle{plainnat}
\bibliography{references.bib}

\end{CJK}
\end{document}

%% file: preamble.tex
\ifdefined\preamble\else\def\preamble{}

\usepackage{tikz}

\usetikzlibrary{shapes,decorations,arrows,calc,arrows.meta,fit,positioning, patterns}
\tikzset{
    -Latex,auto,node distance =1 cm and 1 cm,semithick,
    state/.style ={ellipse, draw, minimum width = 0.7 cm},
    point/.style = {circle, draw, inner sep=0.04cm,fill,node contents={}},
    bidirected/.style={Latex-Latex,dashed},
    el/.style = {inner sep=2pt, align=left, sloped}
}

\usepackage{amsmath,amsthm,amsfonts,amssymb}
\usepackage[ruled,vlined]{algorithm2e}

\usepackage{thmtools,thm-restate}
\declaretheorem{theorem}
\declaretheorem{lemma}

\declaretheorem{corollary}

\theoremstyle{definition}
\declaretheorem{definition}
\declaretheorem{example}

\newcommand{\doo}{{\it do}}

\DeclareSymbolFont{symbolsC}{U}{txsyc}{m}{n}
\DeclareMathSymbol{\boxright}{\mathrel}{symbolsC}{128}

\def\*#1{\mathbf{#1}}

\theoremstyle{definition}

\newcommand{\xdasharrow}[2][->]{
\tikz[baseline=-\the\dimexpr\fontdimen22\textfont2\relax]{
\node[anchor=south,font=\scriptsize, inner ysep=1.5pt,outer xsep=8pt](x){#2};
\draw[shorten <=3.4pt,shorten >=3.4pt,dashed,#1](x.south west)--(x.south east);
}
}

\def\ddefloop#1{\ifx\ddefloop#1\else\ddef{#1}\expandafter\ddefloop\fi}
\def\ddef#1{\expandafter\def\csname bb#1\endcsname{\ensuremath{\mathbb{#1}}}}
\ddefloop ABCDEFGHIJKLMNOPQRSTUVWXYZ\ddefloop
\def\ddef#1{\expandafter\def\csname c#1\endcsname{\ensuremath{\mathcal{#1}}}}
\ddefloop ABCDEFGHIJKLMNOPQRSTUVWXYZ\ddefloop
\def\ddef#1{\expandafter\def\csname v#1\endcsname{\ensuremath{\boldsymbol{#1}}}}
\ddefloop ABCDEFGHIJKLMNOPQRSTUVWXYZabcdefghijklmnopqrstuvwxyz\ddefloop
\def\ddef#1{\expandafter\def\csname v#1\endcsname{\ensuremath{\boldsymbol{\csname #1\endcsname}}}}
\ddefloop {alpha}{beta}{gamma}{delta}{epsilon}{varepsilon}{zeta}{eta}{theta}{vartheta}{iota}{kappa}{lambda}{mu}{nu}{xi}{pi}{varpi}{rho}{varrho}{sigma}{varsigma}{tau}{upsilon}{phi}{varphi}{chi}{psi}{omega}{Gamma}{Delta}{Theta}{Lambda}{Xi}{Pi}{Sigma}{varSigma}{Upsilon}{Phi}{Psi}{Omega}{ell}\ddefloop

\makeatletter
\newcommand*{\indep}{%
  \mathbin{%
    \mathpalette{\@indep}{}%
  }%
}
\newcommand*{\nindep}{%
  \mathbin{
    \mathpalette{\@indep}{\not}
  }%
}
\newcommand*{\@indep}[2]{%
  \sbox0{$#1\perp\m@th$}
  \sbox2{$#1=$}
  \sbox4{$#1\vcenter{}$}
  \rlap{\copy0}
  \dimen@=\dimexpr\ht2-\ht4-.2pt\relax
  \kern\dimen@
  {#2}%
  \kern\dimen@
  \copy0 
} 
\makeatother
\fi